\documentclass[twoside]{article}
\usepackage[accepted]{aistats2021}

\usepackage[round]{natbib}

\usepackage{xr-hyper} 

\usepackage{microtype}
\usepackage{graphicx}
\usepackage{subfigure}
\usepackage{booktabs} 
\usepackage{dsfont}
\usepackage{amsmath, bm, amsthm}
\usepackage{MarkMathCmds}
\usepackage{floatrow}
\usepackage{multirow}
\usepackage{multicol}
\usepackage{tikz}

\newfloatcommand{capbtabbox}{table}[][\FBwidth]
\newtheorem{theorem}{Theorem}

\usepackage{pgfplots}

\pgfplotsset{compat = 1.15}
\pgfplotsset{axis y line=left}
\pgfplotsset{axis x line=bottom}
\pgfplotsset{x tick label style={font=\scriptsize}}
\pgfplotsset{y tick label style={font=\scriptsize}}
\pgfplotsset{x label style={font=\footnotesize}}
\pgfplotsset{y label style={font=\footnotesize}}
\pgfplotsset{legend style={font=\footnotesize}}
\pgfplotsset{title style={font=\small}}
\pgfplotsset{x label style={at={(axis description cs:1.03,0.03)}, anchor=south}}
\pgfplotsset{y label style={at={(axis description cs:0.03,1.03)}, rotate=-90, anchor=west}}

\newcommand{\sinc}{\mathrm{sinc}}
\newcommand{\block}{B}
\newcommand{\f}{\nu}

\newcommand{\fo}{\mu}
\newcommand{\power}{S}

\newcommand{\rr}{r} 

\newcommand{\corr}{\rho} 

\begin{document}

%
\runningtitle{The Minecraft Kernel}

%

\twocolumn[

\aistatstitle{The Minecraft Kernel: \\ Modelling correlated Gaussian Processes in the Fourier domain}

\aistatsauthor{Fergus Simpson \And Alexis Boukouvalas \And  Vaclav Cadek \And  Elvijs Sarkans \And Nicolas Durrande}


\aistatsaddress{Secondmind \And  Secondmind \And Secondmind \And  Secondmind \And Secondmind} 
]

\begin{abstract}
In the univariate setting, using the kernel spectral representation is an appealing approach for generating stationary covariance functions. 
However, performing the same task for multiple-output Gaussian processes is substantially more challenging. We demonstrate that current approaches to modelling cross-covariances with a spectral mixture kernel possess a critical blind spot. 
For a given pair of processes, the cross-covariance is not reproducible across the full range of permitted correlations, aside from the special case where their spectral densities are of identical shape. 
We present a solution to this issue by replacing the conventional Gaussian components of a spectral mixture with block components of finite bandwidth (i.e.\ rectangular step functions).
The proposed family of kernel represents the first multi-output generalisation of the spectral mixture kernel that can approximate any stationary multi-output kernel to arbitrary precision. 
\end{abstract}

\section{INTRODUCTION}
\label{introduction}
Gaussian Processes (GPs) provide a principled and powerful framework for building statistical models~\citep{rasmussen2006}. 
Given some input/output tuples $\{x_i, y_i\}_{i=1}^n$ where $x_i\in \mathcal{X}$ and $y_i\in \mathds{R}$, a GP model typically consists of two elements: a latent GP $f$ that maps input points $x\in \mathcal{X}$ into a latent space, and an observation model (sometimes referred to as the likelihood) that describes the link between $f(x_i)$ and the observation $y_i$. 
The variety of possible choices for the distribution of $f$ and for the observation model make GP models extremely versatile. For example, the choice of the observation model can account for  regression problems with outliers by using heavy tail likelihoods, but it can also be leveraged to tackle different classes of problems such as classification~\citep{nickisch2008approximations} or point process modelling~\citep[see][and references therein]{john2018large}. On the other hand, a large number of mean and covariance functions are available off-the-shelf and they can be used to encode various prior beliefs or assumptions about the system that generated the data, such as the order of differentiability, (non)-stationarity or (quasi)-periodicity. Using the right covariance function comes with massive benefits in term of model accuracy and uncertainty quantification~\citep[see][chapter 5]{rasmussen2006}. 

Three main approaches can be distinguished when it comes to finding a good kernel for the problem at hand. The first one consists of gathering expert knowledge on the phenomenon that generated the data and to define bespoke covariances to encode it \citep{lopez2019physically}, the second is to search automatically through the space of possible kernel combinations \citep{bach2009high, duvenaud2014automatic} and the third consists of choosing a kernel parametrised by a large number of variables such that it is flexible enough to be a good match for a broad range of datasets. This third approach has recently yielded state of the art accuracy on various benchmarks \citep{wilson2013gaussian, sun2018differentiable} and will be the focus of the present work.

The spectral mixture (SM) kernel \citep{wilson2013gaussian, remes2017non} is a canonical example of these highly flexible and heavily parametrised covariance functions. It represents a kernel's spectral density as a mixture of Gaussians, which ensures that the kernel itself can be expressed in closed form. One striking advantage of this approach is that any stationary kernel can be well approximated by the model, provided that the mixture of Gaussians is comprised of a sufficient number of components. An alternative spectral representation based upon a mixture of rectangular blocks of constant spectral density has recently been investigated by \citet{tobar2019band}. This latter approach and its multi-output generalisation will prove to be of particular interest for the current work.

In many applications, it is desirable to use GPs with multivariate outputs. For example if one wants to predict time series corresponding to the temperatures of various neighbouring cities, it is valuable to build a joint model that can account for the correlations between the observations \citep{parra2017spectral}. This complicates the choice of the covariance function because one has to select not only a covariance function for each time series, but one must also specify all cross-covariances between the different time-series. 
For a detailed review of the classic approaches for defining multi-output GPs (Linear Model of Coregionalisation, convolution GPs, etc.), see \cite{alvarez2011kernels}. Generalisations of the SM kernel to a multi-output framework have been proposed by \citet{ulrich2015gp} and \citet{parra2017spectral}. Their approaches consist of modelling the cross-spectral densities between pairs of processes. As we shall demonstrate in this article, such an approach strips the SM kernel of its key ability: to approximate any stationary covariance to arbitrary precision.

The central contribution of this work is to resolve this limitation. We begin by highlighting the root cause of the limitation: is arises as a result of the inevitable overlap between the spectral densities of Gaussian components.  This pathology can  be eliminated by utilising component kernels which possess a finite bandwidth, and can therefore be arranged in a non-overlapping manner. We prove that our generalisation of the sinc kernel \citep{yao1967applications, tobar2019band} to the multi-output case is the first kernel in the spectral mixture class that is dense in the space of multi-output stationary covariances. The practical implication is that the correlations available to the model now spans the full theoretical range. 
This article focuses on a clear exposition of current limitations and how they can be overcome by combining recent work from the literature. Finally we show in the experiments that the inductive bias of the proposed kernels leads to improved performances compared to the conventional form of  multi-output spectral mixture kernels. 


In \S~\ref{sec:background} we review the theoretical background behind univariate and multivariate spectral mixture kernels and highlight the limitations of using a mixture of Gaussians to model cross-spectra. A solution based upon components of finite bandwidth is presented in \S~\ref{sec:solution}, and we show that the proposed method can accurately approximate any stationary multi-output covariance. Experiments are presented in \S~\ref{sec:experiments} where we illustrate the use of the proposed approach for modelling non-stationary time series and for producing  colour images. Finally, concluding remarks are presented in \S~\ref{sec:conclusions}.

Our implementation of the proposed method is based on the GPflow framework~\citep{de2017gpflow}, and the code for generating all of the figures in the paper is available as supplementary material.
\section{SPECTRAL KERNELS} \label{sec:background}
Covariance functions of real-valued processes are functions $K: \ \mathcal{X} \times \mathcal{X} \rightarrow \mathds R$ that are symmetric $K(x, x') = K(x', x)$ and positive definite: $ \sum_{i,j=1}^n a_i a_j K(x_i, x_j) \geq 0$, $\forall n \in \mathds{N}, \forall x_i \in \mathcal X, \forall a_i \in \mathds R$. Checking that a function is positive definite may seem a daunting prospect but for stationary covariances (kernels that can be represented as univarite functions $K(x, x') = K(x-x')$ using the classic notation overloading) a simpler condition is given by Bochner's Theorem~\citep{wendland2004scattered}:
\begin{theorem}[Bochner's theorem]
An integrable function $K(\cdot)$ is the covariance
function of a weakly-stationary real-valued stochastic process if and only if it admits the representation
\begin{equation}
K (\bm{r}) = \int_\mathds{R} e^{i\bm{\f}^\top \bm{r}} S (\bm{\f}) d\bm{\f}
\end{equation}
where $S:\ \mathds{R}^N \rightarrow \mathds{R}$ is integrable, symmetric $S(\bm{\f}) = S(-\bm{\f})$, and positive $S(\bm{\f}) \geq 0$. 
\end{theorem}
The function $S$ is usually referred to as the \emph{power spectrum} or \emph{spectral density}.
%
%
%
Bochner's theorem makes it an appealing basis for defining new covariances: the full family of stationary kernels becomes readily available by exploring positive spectral densities, without having to check any positive definiteness condition. This is the approach  followed by~\cite{wilson2013gaussian} for defining the spectral mixture kernel.

\subsection{Univariate Spectral Mixture kernels} \label{sec:SM}

When modelling the spectral density of the kernel, it is advantageous to choose a parameterisation such that its Fourier transform, i.e.\  the kernel itself, is available in closed form. This increases both the speed and precision with which the marginal likelihood can be evaluated. \citet{wilson2013gaussian} proposed to use spectral densities given by the sum of $Q$ Gaussian pairs:
\begin{align*} \label{eq:component}
\power(\f) =  \sum_{i=1}^Q \frac{A_i}{2} \left[G(\f, \mu_i, \sigma_i) + G(-\f, \mu_i, \sigma_i) \right] \, ,
\\
G(\f, \mu, \sigma) = \frac{1}{  \sqrt{2 \pi} \sigma} \exp \Big(- \frac{ \left( \f - \mu \right)^2}{2 \sigma^2}\Big) \, , 
\end{align*}
where the kernel parameters $\theta=\{ A_i, \mu_i, \sigma_i\}_{i=1}^Q$ satisfy $A_i \geq 0$ and $\sigma_i > 0$  for all $i$. This formalism is readily expanded to operate in higher dimensions by replacing the univariate normal distributions by their multivariate counterparts $G(\bm{\f}, \bm{\mu}_i, \Sigma_i) = \mathcal{N}(\bm{\f} | \bm{\mu}_i, \Sigma_i)$ where the covariance matrix is diagonal. The corresponding kernel, which we will refer to as the \emph{Gaussian-SM}, is
\begin{equation} \label{eq:spectral_mixture}
K(\bm r) =  \sum_{i=1}^Q  A_i \exp(- 2 \pi^2 \bm{r}^\top \Sigma_i \bm{r}) \, \cos( 2 \pi \bm{r}^\top \bm{\mu_i} )  \, .
\end{equation}
%
Gaussian-SM are very flexible and as outlined in \cite{wilson2013gaussian} one of their key property is that they can approximate (in the $L^1$ sense) any stationary covariance. 

\citet{tobar2019band} proposed to replace the Gaussian components in the spectral mixture kernel by block components, motivated by the potential applications relating to signal processing. As with the Gaussian case, it is straightforward to generalise this `block-SM' kernel to input spaces of dimension $D$ by taking the product of $D$ one-dimensional blocks:
\begin{equation}  \label{eq:pk_minecraft}
\begin{split}
    \power(\bm{\f}) &= \sum_{i=1}^Q \frac{A_i}{2}  [ \block_{\bm{\mu}_i, \bm{w}_i}(\bm{\f}) + \block_{-\bm{\mu}_i, \bm{w}_i}(\bm{\f})]
    \\
    \block_{\bm{\mu}, \bm{w}}(\bm{\f}) &=
    \begin{cases}
      \prod_k \frac{1}{w_k} & \text{if all}\ |\f_k - \fo_k| <   \frac{1}{2}w_k \\
      0 & \text{otherwise}
    \end{cases} 
\end{split}
\end{equation}
%
which results in 
\begin{equation}
 K(\bm r) = \sum_{i=1}^Q A_i  \, \cos(2 \pi \bm{r}^\top  \bm{\fo}_i) \prod_{d=1}^D \sinc(r_d \bm{w}_{id}) \, ,
\end{equation}
 where $\displaystyle \sinc(x) \equiv \frac{\sin(\pi x)}{ \pi x}$.

\subsection{Multi-Output Spectral Kernels}
\label{seq:multi_output_SM}
Modelling $N$ outputs (or channels) demands the construction of a model capable of describing not only $N$ spectral densities, but also $\mathcal{O}(N^2)$ distinct cross-spectra. The first multi-output version of SM kernels consisted of using the Linear Model of Coregionalisation (LMC) approach, where each channel is defined as a linear combination of $R$ independent processes with SM kernels \citep{wilson2014covariance}. This has been generalised by \citet{ulrich2015gp} who included phase shifts between the various channels. 
%
Finally, \citet{parra2017spectral} further refined the previous proposals by using the generalisation of Bochner's theorem for multi-output processes:
\begin{theorem}[Cram\'er's Theorem]
A family $\{K_{ij}\}_{i,j=1}^N$ of integrable functions is the kernel of a weakly-stationary multivariate stochastic process if and only if they admit the representation
\begin{equation}
K_{ij} (\bm{r}) = \int_\mathds{R} e^{\mathrm{i}\bm{\f}^\top \bm{r}} S_{ij} (\bm{\f}) d\bm{\f} \qquad \forall i, j \in \{1,\ \dots,\ N\}
\end{equation}
where $S_{ij}:\ \mathds{R}^N \rightarrow \mathds{C}$ are integrable functions that fulfil the positive definiteness condition pointwise
\begin{equation}
\sum_{i,j=1}^N \bar{z}_i z_j S_{ij} (\bm \f) \geq 0 \qquad  \forall z_1,\ \dots,\ z_N \in \mathds{C}, \forall \bm \f \in \mathds{R}^N.
\end{equation}
\end{theorem}
The kernel obtained by~\citet{parra2017spectral} is more general than the previous proposals since it can account for both phase shifts and delays between channels. In all three cases, the formalism ensures that each channel takes the form of a stationary GP with a conventional Gaussian-SM kernel.
%

Despite these increasing levels of refinement in the modelling of multi-output kernels in the spectral domain, the following section will show that none of the aforementioned methods can represent the full range of cross-correlations which a given pair of channels is capable of exhibiting.

\begin{figure*}
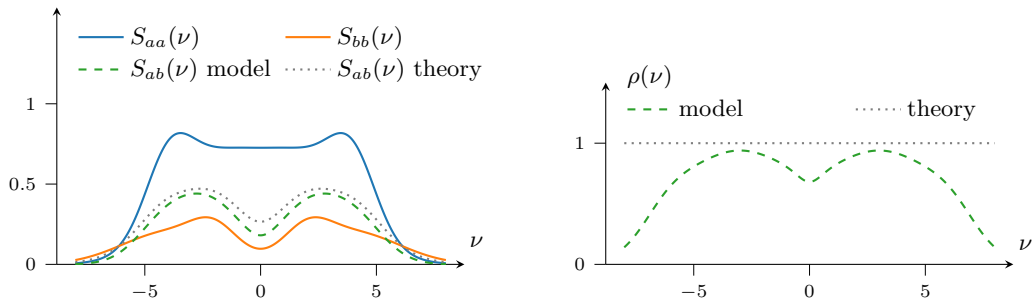
 
\centering 
\input{figures/gaussian_cross_spectrum.tex} \qquad
\input{figures/gaussian_coeff.tex}
\caption{An illustration of the limitation of Gaussian-SM multi-output covariances. Left: The kernel's maximal cross spectrum (dashed line) falls short of the theoretical limit (dotted line) across the whole frequency spectrum. Right: This shortcoming means that possible correlations between the spectral modes are truncated to a reduced range, only permitted to reach the dashed line, and far below the theoretical limit.}  \label{fig:demo}
\end{figure*}

\subsection{The case of missing cross-covariance}
\label{sec:universality}

Arguably the most valuable characteristic of the Spectral Mixture kernel lies in its ability to mimic any stationary kernel. We will now show that this property has been lost in the previous attempts to generalise the SM formalism to multivariate processes. To highlight the crux of the problem at hand, it is instructive to walk through a minimal worked example. We shall explore the case of a GP where the two channels, say $a$ and $b$, 
have a Gaussian-SM kernel with two components: $S_{a} = C_{a_1} + C_{a_2}$ and $S_{b} = C_{b_1} + C_{b_2}$, with the notations $C_{a_1} = \frac{1}{2}A_1 (G_{a_1}^+ + G_{a_1}^-)$ and $G_{a_1}^\pm = G(\cdot, \pm \mu_{a_1}, \sigma_{a_1})$. There is no time delay or phase shift to be concerned about in this example, so the three approaches reviewed in \S~\ref{seq:multi_output_SM} are now equivalent. The resulting cross-spectrum, as given in \citet{parra2017spectral}, is
\begin{equation} \label{eq:walk_through_spectrum}
S_{ab} = \sqrt{G_{a_1}^+G_{b_1}^+} + \sqrt{G_{a_1}^-G_{b_1}^-} + \sqrt{G_{a_2}^+ G_{b_2}^+}  + \sqrt{G_{a_2}^- G_{b_2}^-} \, .
\end{equation}
Plugging these expressions into the definition of the correlation coefficient $\corr(\f)$ which measures if the Fourier modes of $a$ and $b$ tend to be in phase with each other yields
\begin{equation} \label{eq:corr_coefficient}
\begin{split}
\corr (\f) &= \frac{S_{ab}(\f)}{ \sqrt{S_{a}(\f) S_{b}(\f)}} 
\\
&= \frac{ \sqrt{G_{a_1}^+G_{b_1}^+} + \sqrt{G_{a_1}^-G_{b_1}^-} + \sqrt{G_{a_2}^+ G_{b_2}^+}  + \sqrt{G_{a_2}^- G_{b_2}^-}}{\sqrt{(C_{a_1} + C_{a_2})(C_{b_1}  + 
C_{b_2})}}\, .
\end{split}
\end{equation}
At any given frequency $\f$, the Cauchy-Schwarz inequality tells us that this coefficient is 1 if the modes of the two processes are maximally correlated. However, given the specific shape that is assumed for $\power_{ab}$ in Eq.~\ref{eq:walk_through_spectrum} this can only arise if the relative contributions from the two components are identical, such that $C_{a_1} / C_{a_2} = C_{b_1} / C_{b_2}$.
%
For this condition to hold across all frequencies suggests that the two spectra must be of exactly the same shape. Aside from that special case of matching spectra, the Cauchy-Schwarz bound cannot be saturated by the model, and so we find that some loss in cross-covariance is inevitable. Naturally, the same limitation also arises for the less commonly encountered case where the pair of processes are anti-correlated.

The cross-spectrum between these two processes is shown in the left hand panel of Figure \ref{fig:demo}. Note the gap between the maximum attainable value of the model cross-spectrum and the theoretical maximal cross-spectrum. The missing power (or equivalently, the lost correlation) in the right panel is due to the summation of the Gaussian components where they are treated \emph{as if they were independent}, when in reality they need not be. There is inevitably some degree of overlap between the tails of any two Gaussian components, and yet correlations which exist between these tails cannot be fully captured by the conventional model. In the event that the two processes are highly correlated, this can prove to be an extremely poor approximation.

The cross-spectral density $S_{AB}$ between two processes A and B at a given frequency $\nu$ can be interpreted as the expectation of the product of their Fourier modes $\langle f_A(\nu),  f_B^*(\nu) \rangle$. The Cauchy-Schwarz inequality tells us that $| \langle u, v \rangle |^2 \leq \langle u, u \rangle \langle v, v \rangle$. The theoretical range of a cross spectral density is therefore given by $S^2_{AB} \leq  S_{AA} S_{BB}$, and the dotted  line in the figure represents the regime where this inequality is saturated. This corresponds to the Fourier components being perfectly in phase. 

As the overlap between a given pair of components grows, so too does the potential loss in cross-covariance experienced by a Gaussian-SM kernel. This highlights the danger of using broadband components, those which span large frequencies due to a high bandwidth $\sigma$, in a multi-output setting. 

Note that introducing more than two components does not resolve the problem - it exacerbates it. This is because a process described by $N$ spectral components generates $\mathcal{O}(N^2)$ pairs of overlapping Gaussians, yielding greater opportunities for the cross-covariance to be lost.

\section{A UNIVERSAL MODEL FOR CROSS SPECTRAL DENSITIES} \label{sec:solution}

In this section we shall specify a GP kernel which permits a full exploration of the space of possible cross-covariances, while ensuring that the number of model parameters grows linearly with the number of components. This constitutes the central contribution of this work. 

\subsection{The Minecraft kernel}
Since the limitation of the multi-output Gaussian-SM comes from the overlap between the tails of the components, the issue can be resolved by selecting components with disjoint support. By replacing the pairs of Gaussians with pairs of `blocks' (i.e.\ rectangular functions, as defined in \S~\ref{sec:SM}) we can reap two major benefits. First of all, it will allow a complete description of the cross-covariance for highly correlated processes, resolving the issue highlighted in the previous section, thereby allowing any multi-output stationary kernel to be well approximated. Secondly, it becomes easier to evaluate the coefficients which determine the cross-covariances, using a technique which is outlined in the following subsection. 
\begin{figure*}[t]
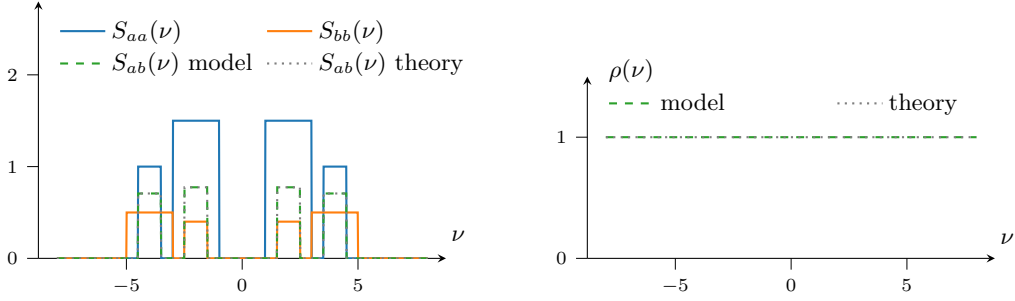

\centering
\input{figures/block_cross_spectrum.tex}
\qquad 
\input{figures/block_coeff.tex}
\caption{The introduction of spectral blocks resolves the issue highlighted in Figure \ref{fig:demo}. The model cross-spectrum can now be constructed from non-overlapping components, allowing the cross-spectrum to replicate the target. All Fourier modes can now be fully correlated, so there is no loss in cross-covariance.}  
\label{fig:cross_spectra_blocks}
\end{figure*}
This motivates the introduction of the \emph{Minecraft kernel} whose spectral representation is a sum over $Q$ symmetrised block components as defined in (\ref{eq:pk_minecraft}):
\begin{align} \label{eq:minecraft_kernel}
\power_{ij}(\bm{\f}) &= \sum_{q=1}^Q \frac{1}{2} A_{ij}^q (B_{\bm{\mu}^q,\bm{w}^q}(\f) + B_{-\bm{\mu}^q,\bm{w}^q}(\f)) \\ 
K_{ij}(\bm{\rr}) &= \sum_{q=1}^Q A_{ij}^{q} \cos( \bm{\rr}^\top  \bm{\fo}^{q} ) \prod_{d=1}^D  \sinc \left( \rr_d  w_d^{q}   \right)    \, .  \label{eq:minecraft_k}
\end{align}

Provided that the $Q$ amplitude matrices $\{A^q_{ij}\}_{i,j=1}^N$ are positive definite, Cram\'er's theorem guarantees that the Minecraft kernel is a valid covariance function. Note that the Minecraft Kernel can either be interpreted as the multi-output generalisation of the block-SM kernel~\citep{tobar2019band}, or as a Riemann approximation of the integral in Cram\'er's theorem if the $S_{ij}$ are seen as a constant per block functions. This remark suggests that the Minecraft kernel can recover the universal approximation property that had been lost in previous multi-output SM kernels. This is guaranteed by the following result:

\begin{theorem} \label{thm:universality}
Minecraft kernels are dense in the space of multi-output stationary real-valued covariance functions for the $L^1$ norm.
\end{theorem} 
\begin{proof}
Let $K = \{K_{ij}\}_{i,j=1}^N$ be the covariance of weakly-stationary real-valued multivariate stochastic process, and let $\{S_{ij}\}_{i,j=1}^N$ be the associated spectral density given by Cram\'er's theorem. Finally, let $R_{\bm \mu, \bm w}$ denote the hyper-rectangle in $\mathds{R}^D$ centred on $\bm \mu$ with width $w_i$ along the $i$th axis.
Since simple functions are dense in $L^1$ \citep{bogachev2007measure}, and since $S_{ij}(\bm \f) = S_{ij}(- \bm \f)$, we can find a sequence of sets of non-overlapping rectangles $(\{R_{\bm{\mu}_{kl},\bm{w}_{kl}}\}_{l=1}^{L_k} \cup \{R_{-\bm{\mu}_{kl},\bm{w}_{kl}}\}_{l=1}^{L_k})_{k \geq 0}$ such that for all $i,j \in \{1,\ \dots, \ N\}$
\begin{equation} \label{eq:proof}
\begin{split}
\bigg(\sum_{l=1}^{L_k} \int_{R_{\bm \mu_{kl}, \bm w_{kl}}} \hspace{-8mm}S_{ij}(\bm{\f})  d\bm \f \, 
\big(B(\cdot ,\, &\bm{\mu}_{kl}, \bm{w}_{kl}) \,+\\[-3mm]
& B(\cdot, - \bm{\mu}_{kl}, \bm{w}_{kl})\big)  \bigg)_{k \geq 0}
\end{split}
\end{equation}
is a sequence of constant per block functions that converges to $S_{ij}$ in $L^1$. Let $S_{ij}^k$ denote the elements of these sequences, we will now prove that they satisfy the two conditions of Cram\'er's theorem. First, these functions are clearly integrable. Second, for some given $k$ and $\bm \f$, we have either $\bm \f \notin \bigcup_l R_{\pm \bm \mu_{kl}, \bm w_{kl}}$ which implies $S_{ij}^k (\bm \f)=0$ for all $i,j$, or there exists a unique $l$ such that $\bm \f \in R_{\pm \bm \mu_{kl}, \bm w_{kl}}$ and $S_{ij}^k (\bm \f)= \frac{1}{|R_{\bm \mu_{kl}, \bm w_{kl}}|} \int_{R_{\bm \mu_{kl}, \bm w_{kl}}} S_{ij}(\bm \f) d \bm \f$ where $|\cdot|$ is the area of the rectangle. In both cases, $(S_{ij}^k (\bm \f))_{i,j=1}^N$ is positive definite: either because a matrix of zero is trivially positive definite, or because the set of positive definite matrices is a convex cone and $S_{ij}^k (\bm \f)$, is the rescaled integral of positive definite matrices. Cram\'er's theorem thus applies and tells us that the Fourier transform of $S_{ij}^k$ is the covariance function of a weakly-stationary multivariate process which by definition is a Minecraft kernel. By continuity of the Fourier transform, this sequence of Minecraft kernels converges to $K_{ij}$ for the $L^1$ norm.
\end{proof}
\begin{figure*}
    \centering
    \includegraphics[width=5cm]{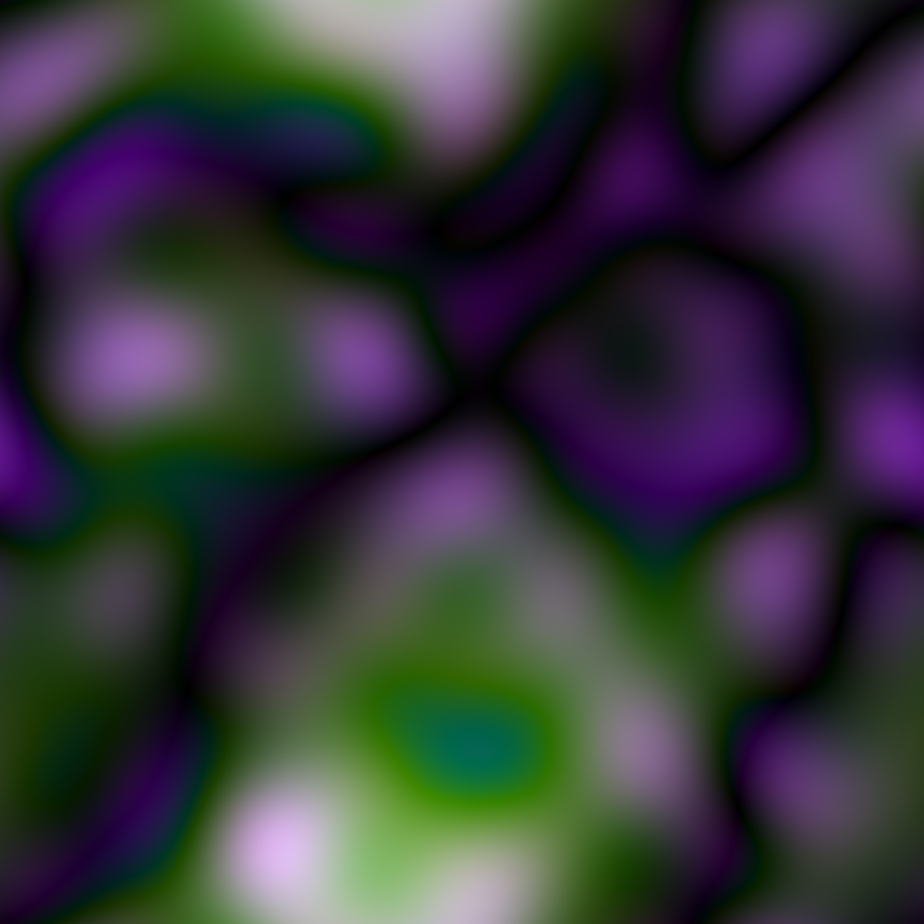} 
    \qquad
    \includegraphics[width=5cm]{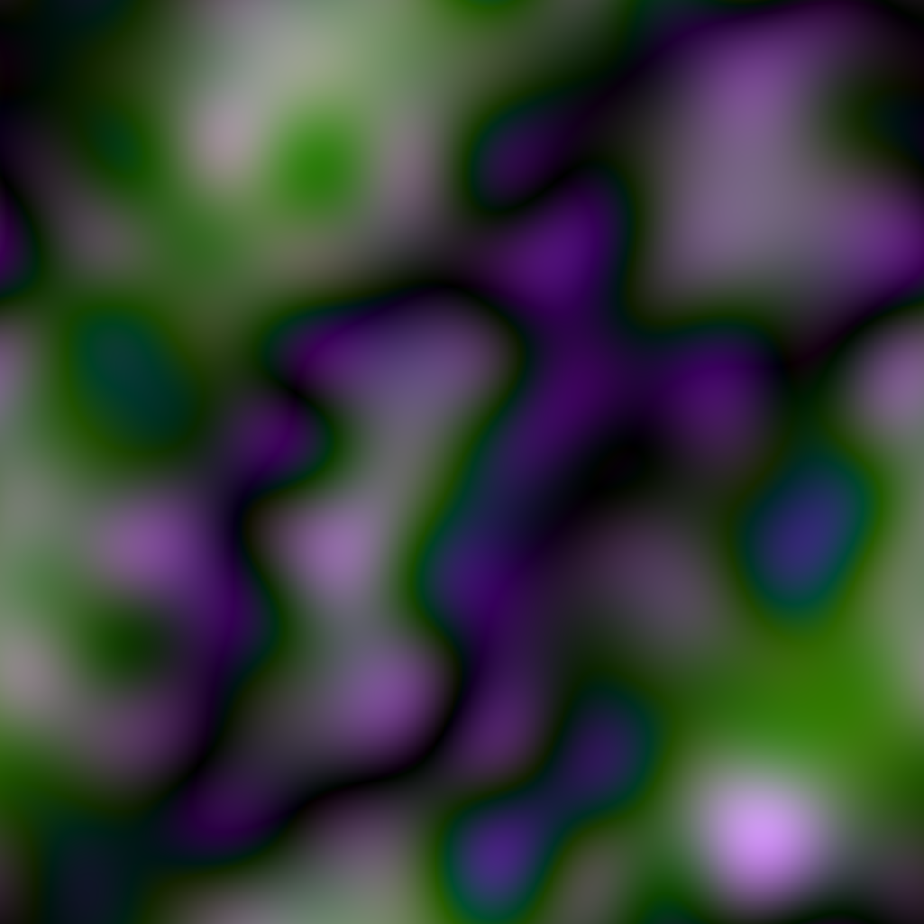}
    \qquad 
    \includegraphics[width=5cm]{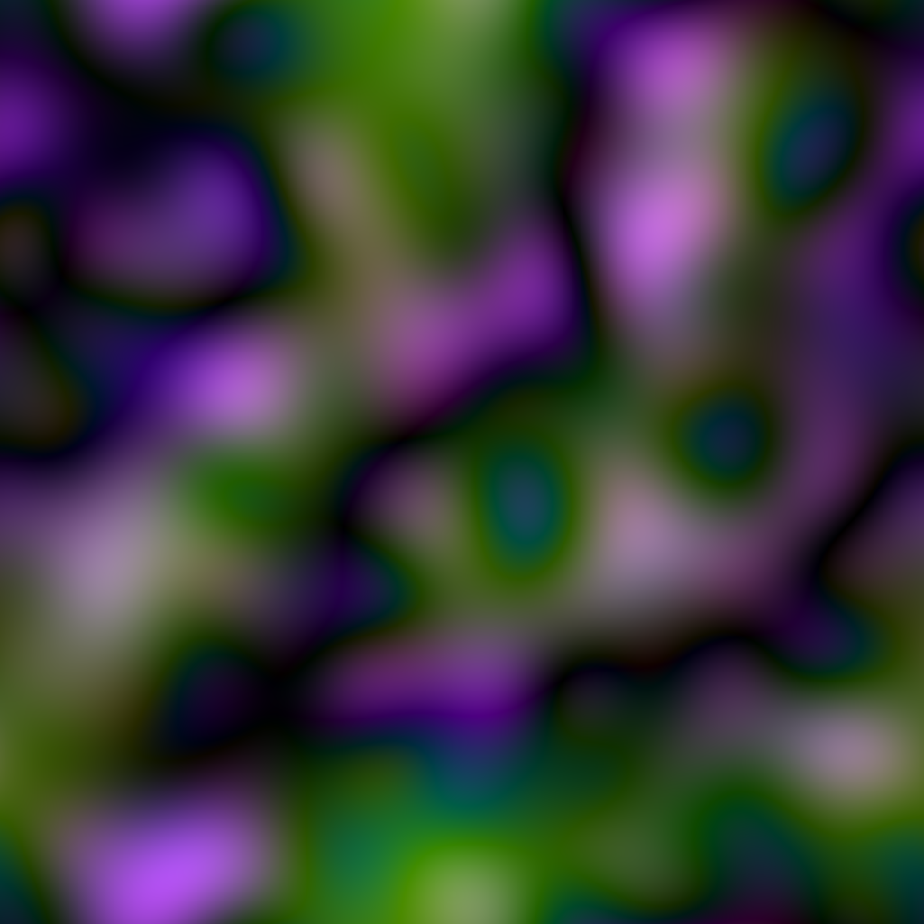} 
    \caption{A comparison of images where the RGB channels are determined by drawing samples from a multi-output Gaussian process. Left: A sample from the target kernel we wish to model, comprising three Matern kernels, which possess correlations given by Table (\ref{tab:correlations}). Middle: The GP has a Minecraft kernel which can ensure a high correlation between the three channels. Right: This GP has a Gaussian-SM covariance, and the limited cross-correlation range results in a substantial loss of coherence between the channels. }
    \label{fig:coherence}
\end{figure*}

Since the Minecraft kernel can be interpreted as a Linear Model of Coregionalisation, one corollary of Theorem~\ref{thm:universality} is that Linear Models of Coregionalisation can approximate any stationary multi-output covariance to arbitrary precision.

The ability of the Minecraft kernel to model highly correlated channels is illustrated in Figure~\ref{fig:cross_spectra_blocks}. Contrary to the case where Gaussian components are used, components can be arranged in a non-overlapping manner, This configuration enables the cross-covariances to now reach the theoretical limit. This can be seen explicitly in the right hand panel, as indicated by a correlation coefficient which saturates at unity across the full range of frequencies.

If we wish to generalise the model to incorporate a possible delay between the different processes, we can adopt the prescription advocated by \citet{parra2017spectral} to yield
\begin{equation} \label{eq:delayed_minecraft_kernel}
\begin{split}
K_{ij}(\bm{\rr}) &=  \sum_{q=1}^Q A_{ij}^{q} C_{ij}^{q} S_{ij}^{q} \, ,  
\\
C_{ij}^{q} &= \cos\left[(\bm \rr + \theta^{q}_{ij})^\top \bm \mu^{q} + \phi^{q}_{ij} ) \right] \, ,  
\\
S_{ij}^{q} &= \prod_{d=1}^D  \sinc \left[ (\rr_d  +\theta^{q}_{ij}) w_d   \right]  \, .
\end{split}
\end{equation}

\subsection{Block shapes}

In principle, the blocks we use to model the spectral density need not be rectangular. Indeed we note that, when working in higher input dimensions, it may prove beneficial for each spectral component to take the form of a pair of ellipsoids. This would allow the potentially long product of sinc functions in equation (\ref{eq:minecraft_k}),  one for each input dimension, to be replaced by a single Bessel function: 
%
\begin{equation} \label{eq:bellipsoidal_kernel}
K_e(\bm{r}) =\sum_{q=1}^Q A_{ij}^{q} ||\bm r \odot \bm w ||^{n/2}  \, \cos(2 \pi \bm{r}^\top \bm{\fo})  J_{n/2}(||\bm r \bm \odot \bm w||)  \, ,
\end{equation}
where $|| \cdot ||$ is the Euclidean norm and $\odot$ denotes an element-wise product.  This kernel retains the desired property of the minecraft kernel, that the components are of finite bandwidth.  


\subsection{Parameterisation}

Two remarks can be made regarding the parameterisation of the Minecraft kernel. As with the multi-output Gaussian-SM, the parameters $\bm{w}^q$ and $\bm{\fo}^q$ do not depend on $i, j$ (which means that all spectral densities share the same blocks as basis functions), and that we want the blocks' support to be non-overlapping.

These choices to not imply any loss of generality (they may require using a larger number of blocks) but they confer two benefits. The first one is to guarantee that no cross-covariance is lost and that the maximum correlation between the signal can be reached. The second is to ensure that each component cannot interact with more than one block per channel, which implies that the number of free parameters in the model grows only as $\mathcal{O}(Q)$.

A final remark on parametrisation is each $A_{ij}^q$ matrix contains $\mathcal{O}(D^2)$ model parameters that typically need to be optimised, but it must at the same time satisfy the positive definiteness constraint. A convenient method is to reparametrise them by their Cholesky factors before exposing the latter to the optimiser. If this $\mathcal{O}(D^2)$ scaling in the number of parameters cannot be afforded, a low rank decomposition of $A_{ij}^q$ may be used instead.

\section{EXPERIMENTS} \label{sec:experiments}
In this section we illustrate the advantages of the proposed approach on two case studies that require the cross-covariances to be accurately accounted for. The first one involves modelling color channels in an image, and the second revisits the use of change points for modelling non-stationary time series.

\subsection{Rendering of images} \label{sec:images}
Colour images are a form of multi-channel data in which there tends to be a significant degree of correlation between the three RGB channels. While some distinct information is invariably carried within each channel, a simple modulation in brightness across an image will be shared across all three of them.

In this section, we aim to generate images which possess a significant degree of correlation between the three channels, so as to generate coherent fluctuations in brightness. As a baseline we select a different covariance for each R, G and B channel corresponding to isotropic 2D Mat\'ern kernels, with regularity 1/2, 3/2, and 5/2 respectively.
 \begin{table}[h]
\caption{Correlations between the RGB channels presented in Figure \ref{fig:coherence}, and the relative errors compared to the target correlations. The values obtained with Minecraft are more than an order of magnitude closer to the target correlations than those from the conventional Gaussian model.}
 \begin{center}
 \begin{tabular}{clccc}
 \toprule
        \multicolumn{2}{l}{Channels} & $R, G$    & $R, B$ & $G, B$ \\ \midrule
 \parbox[t]{2mm}{\multirow{3}{*}{\rotatebox[origin=c]{90}{Correl.}}} & Target   & 0.897 & 0.960 & 0.744 \\
  & Gaussian & 0.849  & 0.908  & 0.566 \\
 & Minecraft & 0.901  & 0.962  & 0.754 \\ \midrule
 \parbox[t]{2mm}{\multirow{3}{*}{\rotatebox[origin=c]{90}{Error}}} & Target   & - & - & - \\
  & Gaussian & -5.4\% & -5.4\% & -23.9\%\\
 & Minecraft & +0.4\% & +0.2\% & +1.2\%\\ \bottomrule
 \end{tabular}
 \end{center}
 \label{tab:correlations}
 \end{table}

Given this fiducial model, we shall attempt to approximate the covariance using both variants of the multi-output SM kernels: the conventional Gaussian components and the Minecraft kernel. In each case, we tile the area $[-2, 2]^2$ of the spectral domain with 64 basis functions, which means that we use 32 components. The amplitude of each component is obtained by minimising the $L^1$ distance between the original Mat\'ern spectrum and the approximations.

The images shown in Figure~\ref{fig:coherence} are obtained by sampling jointly from the channels of the three GPs. In the left hand panel  we see a sample drawn from the correlated Mat\'ern kernels, in the centre panel is a sample drawn from a Minecraft kernel, and in the right hand panel is a sample is drawn from  a conventional Gaussian-SM kernel. The brightness of a given channel - red, green, or blue - is determined by the absolute magnitude of the sample at the given pixel location. Black contours therefore correspond to the regime where the sample crosses zero. Changes in colour correspond to uncorrelated, decoherent fluctuations across the three channels. Meanwhile changes in brightness correspond to correlated, coherent fluctuations across the three channels. More correlated processes will therefore tend to produce fewer changes in colour but more pronounced changes in brightness.

In order to gain a more quantitative measure of performance, beyond this visual illustration, we provide a numerical comparison of the correlations between the channels for the three different kernels. Since the kernels are all stationary, the cross-correlations do not depend on the input location. We can therefore pick 
an arbitrary input point and compute the correlations between the three channels. As shown in Table~\ref{tab:correlations}, the Minecraft formalism offers a much more accurate account of the cross-channel dependencies. 
\begin{figure}[ht]
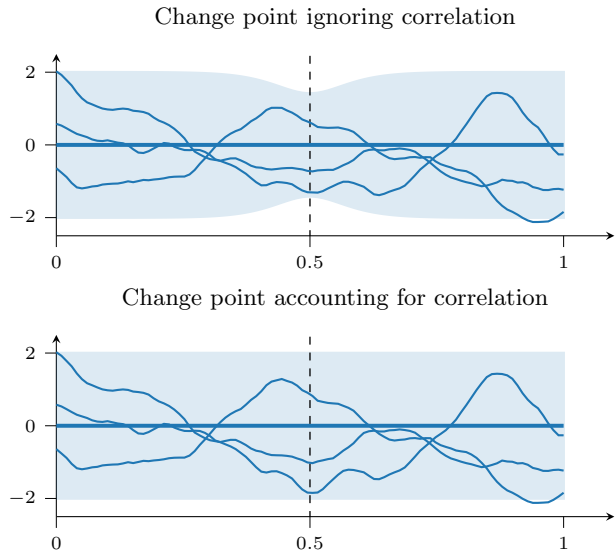

\begin{center}
\input{figures/change_point.tex}
\input{figures/transition_point.tex}
\caption{Samples drawn from a Gaussian process with a change point located at $0.5$ (vertical dashed line). Top: If $f_1$ and $f_2$ are taken to be independent, some variance is lost (the shaded confidence intervals can be seen to contract around the change point location. Bottom: There is no loss of variance when the cross-correlation between $f_1$ and $f_2$ is accounted for. }
\label{fig:matching_transition}
\end{center}
\end{figure}

We also note that the sign of the error is consistently different in the two cases, and this is due to their fundamentally different origins. The Gaussian model under-represents the magnitude of all the correlations, due to its inability to reproduce the full range of possible cross-covariance, as illustrated in Figure \ref{fig:demo}.  In the case of the Minecraft model, it does not possess this limitation, so the leading source of modelling error stems from the accuracy of reproducing the shape of the target spectral density. This is limited by the finite number of components chosen ($Q=32$ in this case). Within the bandwidth of a single component, the spectra of two outputs in the minecraft kernel share the same functional form, and this leads to a slight overestimation of the cross-correlation compared with the target case.

\subsection{Change points in non-stationary time series}

A useful technique for generating a non-stationary process is via a linear combination of stationary processes. For example, given two Gaussian processes of the form $f_1 \sim \mathcal{GP}(0, k_1)$, $f_2 \sim \mathcal{GP}(0, k_2)$ and a sigmoid function $s$, one can define a GP that smoothly transitions from $f_1$ to $f_2$:
\begin{equation} \label{eq:changepoint}
    f(x) = s(x) f_1(x) + (1-s(x)) f_2(x)\, .
\end{equation}
This approach is fairly common in the GP community, it is implemented within the GPflow package, and it serves as one of the building blocks within the automatic statistician \citep{duvenaud2014automatic, lloyd2014automatic}. However, it is typically assumed that the processes $f_1$ and $f_2$ are independent, which results in some of the variance of $f$ vanishing at the transition point. This can be illustrated with a simple experiment where the two kernels we wish to connect are identical, such that $k_1 = k_2$. Since $f_1$ and $f_2$ have the same distribution, one could expect that the global distribution remains unchanged when applying a change point. As seen in the upper panel of  Figure~\ref{fig:matching_transition} this is however not the case, and a significant proportion of the variance is lost around the transition point.

This unwanted behaviour can be addressed by relaxing the assumption that $f_1$ and $f_2$ are independent. On the example detailed above, choosing $f_1 = f_2$ results in the expected behaviour where the distribution of $f$ is exactly the same as the distribution of $f_1$ and $f_2$ (see the lower panel of Figure \ref{fig:matching_transition}). This pedagogical example  highlights the importance of adequately modelling cross-covariances.  

Practitioners may be concerned whether blocks are as efficient as Gaussians at replicating real-world spectral densities.  To assess their capabilities, we study the thirteen benchmark time series used in \cite{lloyd2014automatic}. For each of these time series, we fit models with one change point (as defined in equation \ref{eq:changepoint}) and compare the accuracy of using Gaussian components versus Minecraft's block components. As discussed earlier, the Gaussian-SM cannot fully account for the correlation between the components whereas the block-SM can. In both models, we make use of $Q=10$ components.
Hyperparameters are initialised by selecting the highest marginal likelihood from 1,000 random starting points, before being optimised via SciPy's implementation of the conjugate gradients algorithm for 2,000 iterations.  Finally, the accuracy of the models are evaluated by computing for each time series the Standard Mean Square Error (SMSE). Following \cite{lloyd2014automatic}, we adopt a 90/10 train/test split ratio. 

As seen in Table~\ref{tab:tseries}, the blocks compare as well as, and in many cases better than, their Gaussian counterparts. Quoted uncertainties are estimated by repeating the experiments with ten different random seeds.  It should however be noted that we do not necessarily recommend using change points for these datasets since in some cases better performances are achievable with stationary covariances.


 \begin{table}[h]
\caption{Comparing the SMSE values for benchmark time series when adopting two different choices of spectral component. }
 \begin{sc}
 \begin{center}
 \begin{tabular}{lccc}
 \toprule
 Dataset             & Block                  & Gaussian \\\midrule
 {airline}           &\bf{0.39} ($\pm0.07$)    &0.55 ($\pm0.06$) \\
 {births}            &0.99          ($\pm0.01$)  &\textbf{0.81} ($\pm0.4$) \\
 {call centre}       &\textbf{7.28}  ($\pm3.8$).  &7.83.  ($\pm4.5$)    \\
 {gas production}    &1.46  ($\pm0.24$)         &\textbf{0.57}  ($\pm0.16$)\\
 {internet}         &\textbf{1.00}  ($\pm0.04$) &0.89  ($\pm0.13$) \\
 {mauna}            &0.52  ($\pm0.1$)            &\textbf{0.45}  ($\pm0.05$) \\
 {radio}            &\textbf{1.12}  ($\pm0.08$).  &1.46  ($\pm 0.09$) \\
 {solar}            &\textbf{1.35}  ($\pm0.4$)        &1.50  ($\pm0.16$) \\
 {sulphuric}        &\textbf{5.35}    ($\pm0.13$)   &5.94  ($\pm0.1$) \\
 {temperature}     &\textbf{0.77} ($\pm0.14$)       &0.93  ($\pm0.1$) \\
 {unemployment}    &\textbf{1.21}      ($\pm0.05$)    &1.22  ($\pm0.12$) \\
 {wages}           &2.38  ($\pm0.09$)              &\textbf{2.08}  ($\pm0.12$) \\
 {wheat}           &\textbf{1.19}   ($\pm0.7$)          &1.41  ($\pm0.1$) \\
 \midrule
 mean &\textbf{1.92} ($\pm0.1$) &1.97 ($\pm0.1$) \\
 best performance &\textbf{9/13}  & 4/13  \\
 \bottomrule
 \end{tabular}
 \end{center}
 \end{sc}
 \label{tab:tseries}
 \end{table}

\section{CONCLUSIONS} \label{sec:conclusions}

Modelling the cross-covariances between Gaussian Processes is a challenging but important task in machine learning. We have highlighted a significant blind spot in the conventional approach of modeling cross-covariances in the spectral domain: the important regime where a pair of processes are significantly correlated (either in the positive or negative sense) cannot be adequately modelled via a mixture of Gaussians. Aside from the trivial case where the spectral densities of each process are of the same functional form, overlapping components lead to a loss of cross-covariance. In general, we advise that multi-output spectral kernels adopt different initialisation strategies to the standard single-output case - one which suppresses the generation of very broad components.   


We present a new multi-output spectral kernel which offers a resolution to this problem. By utilising a  basis kernel of finite bandwidth, 
we can avoid the loss of cross-covariance caused by overlapping components. By replacing a conventional mixture of Gaussians with a mixture of blocks of constant spectral density, this kernel opens up access to the full range of cross-covariances associated with stationary processes.  A further key advantage of this approach is that, without loss of generality,  all spectral and cross spectral densities can share the same base parameterisation, which helps restrict the number of free parameters in the model.  Finally, we have also presented and demonstrated a method for combining these correlated processes in order to model non-stationary time series. 

\subsubsection*{Acknowledgements}
We would like to thanks James Hensman for several helpful discussions, and the anonymous reviewers for their constructive feedback.

\bibliographystyle{unsrtnat}
\bibliography{minecraft.bib}

\end{document}